\documentclass[11pt]{article}
\usepackage[margin=1in]{geometry} 

\usepackage{amsmath,amssymb,amsthm}
\usepackage{graphicx}
\usepackage{booktabs}
\usepackage{hyperref}
\usepackage{enumitem}
\usepackage{xcolor}
\usepackage{microtype}
\usepackage{algorithm}
\usepackage{algorithmic}
 \usepackage{natbib}
\usepackage{multirow}
\usepackage{float}
\usepackage{placeins}

\newtheorem{proposition}{Proposition}
\newtheorem{theorem}{Theorem}
\newtheorem{corollary}{Corollary}
\newtheorem{remark}{Remark}
\newtheorem{definition}{Definition}
\newtheorem{lemma}{Lemma}

\title{Proactive Routing to Interpretable Surrogates with Distribution-Free Safety Guarantees}

\author{
  Iqtedar Uddin \\
  Illinois Institute of Technology \\
  \texttt{iuddin1@hawk.illinoistech.edu}
  \and
  Mazin Khider \\
  Illinois Institute of Technology \\
  \texttt{mkhider@hawk.illinoistech.edu}
  \and
  Andr\'e Bauer \\
  Illinois Institute of Technology \\
  \texttt{abauer7@illinoistech.edu}
}
\date{}

\begin{document}
\maketitle

\begin{abstract}
Model routing determines whether to use an accurate black-box model 
or a simpler surrogate that approximates it at lower cost or greater 
interpretability. In deployment settings, practitioners often wish to 
restrict surrogate use to inputs where its degradation relative to a 
reference model is controlled. We study proactive (input-based) routing, 
in which a lightweight gate selects the model before either runs, 
enabling distribution-free control of the fraction of routed inputs 
whose degradation exceeds a tolerance $\tau$.
The gate is trained to distinguish safe from unsafe inputs, and a 
routing threshold is chosen via Clopper--Pearson conformal calibration 
on a held-out set, guaranteeing that the routed-set violation rate is 
at most $\alpha$ with probability $1-\delta$. We derive a feasibility 
condition linking safe routing to the base safe rate $\pi$ and risk 
budget $\alpha$, along with sufficient AUC thresholds ensuring that 
feasible routing exists. Across 35 OpenML datasets and multiple 
black-box model families, gate-based conformal routing maintains 
controlled violation while achieving substantially higher coverage 
than regression conformal and naive baselines. We further show that 
probabilistic calibration primarily affects routing efficiency rather 
than distribution-free validity.
\end{abstract}

\section{Introduction}\label{sec:intro}

Practitioners increasingly face a deployment dilemma: interpretable 
models such as shallow decision trees and linear regressors are 
auditable, fast, and regulatorily favored, but are rejected in favor 
of black-box ensembles because no formal mechanism exists to establish 
when they are accurate enough to use. The intuitive 
solution---deploy the interpretable model when it agrees with the 
black-box, i.e., when $|f(x) - g(x)| \leq \tau$---requires running 
the expensive model on every input, defeating the purpose. We propose 
a proactive solution: identify safe inputs from features alone, before 
either model runs, with formal guarantees on the fraction of routed 
inputs where quality degrades unacceptably. When the interpretable surrogate performs acceptably on a given input, using it saves substantial compute with negligible quality loss. Local explanation methods such as LIME~\citep{ribeiro2016why} and SHAP~\citep{lundberg2017unified} construct surrogates to approximate black-box behavior. Our framework shows when such surrogates can safely replace the black-box rather than merely explain it, addressing the practical tension between black-box accuracy and interpretable deployment emphasized by Rudin~\citep{rudin2019stop}.

Model routing formalizes this choice: given an expensive model $f$ and a cheap surrogate $g$, decide per-input $x$ which to use. Existing approaches include cascading \citep{viola2001rapid, trapeznikov2013supervised, chen2023frugalgpt}, where the cheap model runs first and escalates to the expensive model if uncertain, and disagreement routing \citep{ribeiro2016why}, where the surrogate is used when $|f(x) - g(x)| \leq \tau$. Both approaches are output-dependent. They require running at least one model before deciding. Cascading must run $g$ on every input. Disagreement routing requires $f(x)$, largely defeating the purpose of cost savings. Furthermore, while some frameworks utilize input-based routing to dynamically skip computation for efficiency \citep{wang2018skipnet, wu2018blockdrop}, they rely on empirical heuristics and do not provide formal guarantees on the quality of the routed predictions.

We propose proactive routing with distribution-free safety guarantees. A lightweight classifier we call the \textit{gate} maps input features to a safety score $s: \mathcal{X} \to [0, 1]$ before either model runs. Unlike cascading, where the cheap model must process every input before deciding to escalate, and unlike disagreement routing, where the expensive model's output $f(x)$ is needed, our gate examines only the input $x$ and makes an immediate routing decision.

The gate is trained to distinguish inputs where the surrogate performs acceptably (degradation $d(x) \leq \tau$) from those where it does not.
A conformal calibration procedure then selects a routing threshold $t^*$ from a held-out calibration set, providing a PAC-style guarantee:
\[
  \Pr\!\Big(\Pr(d(X) > \tau \mid s(X) \geq t^*) > \alpha\Big) \leq \delta.
\]
This guarantee is distribution-free, requiring only exchangeability between calibration and test data. It holds regardless of the gate's calibration quality (ECE)~\citep{guo2017calibration}, model family, or data distribution.

\paragraph{Contributions.}
We formalize proactive surrogate routing and derive an exact 
finite-sample bound on routed-set violation via Clopper--Pearson 
conformal calibration. We characterize routing feasibility through 
a likelihood-ratio condition and sufficient AUC thresholds, and 
establish a tight lower bound on achievable coverage under concavity. 
We show that distribution-free validity is independent of probabilistic 
calibration, though calibration affects efficiency. Empirically, across 
diverse tabular benchmarks (LCBench and OpenML), gate conformal routing 
achieves higher coverage with controlled violations compared to 
regression-based baselines.

\FloatBarrier
\section{Related Work}\label{sec:related}

\paragraph{Conformal prediction and risk control.}
Split conformal prediction~\citep{vovk2005algorithmic} constructs distribution-free prediction sets with finite-sample coverage guarantees. Conformal Risk Control~\citep{bates2021distribution} guarantees $\mathbb{E}[L(\lambda)] \leq \alpha$, controlling expected loss. Our framework provides a stronger high-probability guarantee: $\Pr(V(t^*) > \alpha) \leq \delta$, bounding the probability of any exceedance. Expected risk control permits frequent violations on individual deployments. Our guarantee ensures that the violation rate exceeds $\alpha$ with probability at most $\delta$, which is the relevant requirement for safety-critical deployment. Learn Then Test~\citep{angelopoulos2021learn} provides a comparable high-probability guarantee over families of thresholds. Our Clopper--Pearson approach achieves the same guarantee type for nested thresholds with a simpler mechanism, justified by the fixed-sequence testing argument (Remark~\ref{rem:adaptivity}).

\paragraph{Selective prediction and deferral.}
The reject option~\citep{chow1970optimum,elyaniv2010foundations} allows classifiers to abstain on uncertain inputs, trading coverage for accuracy. Selective classification~\citep{geifman2017selective} provides risk-coverage tradeoffs for deep networks. Learning to defer~\citep{madras2018predict} routes inputs to human experts. Our framework generalizes selective prediction: rather than abstaining (discarding the input), we route to an alternative model. The safety constraint is on substitution quality, not prediction correctness. While selective classification controls conditional accuracy above a confidence threshold, we instead control degradation relative to a reference model and replace abstention with an alternative predictor. Although the conditional risk bound has a similar mathematical form, our safety criterion measures surrogate degradation rather than prediction correctness alone.

\paragraph{Model cascading and routing.}
Cascades~\citep{viola2001rapid,trapeznikov2013supervised} run models sequentially, escalating when confidence is low. FrugalGPT~\citep{chen2023frugalgpt} and Hybrid LLM~\citep{ding2024hybrid} cascade through LLM APIs. Mixture-of-experts~\citep{jacobs1991adaptive} routes via learned gating networks. These approaches are largely output-dependent: routing decisions use model outputs (confidence, predictions). Conversely, input-based routing frameworks \citep{wang2018skipnet, wu2018blockdrop} dynamically skip computation by evaluating only the input features $x$ before the main models run, reducing latency and avoiding unnecessary computation on the cheap model for hard inputs. However, these proactive methods rely on empirical heuristics and lack formal guarantees on prediction quality. While an input-based gate requires labeled training data from both models, this is a one-time offline cost amortized over all future inference.

Recent works on model selection with guarantees have successfully addressed prediction quality, but primarily through output-dependent approaches: post-hoc arbiters \cite{overman2025conformal}, output-based selective prediction \cite{wang2025linear, valkanas2025c3po}, and guaranteed sequential cascades \cite{chen2023frugalgpt}. Because these methods require executing at least one model before making routing decisions, they incur the very computational cost proactive routing aims to avoid. By framing proactive routing as an input-only classification task governed by a Neyman--Pearson feasibility condition, we characterize its theoretical limits and provide exact finite-sample safety bounds on the routed subset without requiring prior model execution.

\paragraph{Surrogate fidelity and interpretability.}
LIME~\citep{ribeiro2016why} and SHAP~\citep{lundberg2017unified} 
construct local surrogates to explain black-box predictions, routing 
implicitly when the surrogate agrees with the black-box 
(fidelity-based, i.e., $|f(x)-g(x)| \leq \tau$). This requires 
$f(x)$, which defeats proactive routing. Rudin~\citep{rudin2019stop} 
argues that practitioners default to black-boxes precisely because 
interpretable models lack formal per-instance deployment guarantees. 
Knowledge distillation~\citep{hinton2015distilling} trains surrogates 
globally but provides no such guarantee. We close this gap: our safety 
criterion is degradation relative to ground truth rather than agreement 
with the black-box, ensuring the surrogate is genuinely accurate on routed inputs—providing a formal safety characterization that complements Rudin's call for interpretable models in high-stakes settings.

\paragraph{Calibration.} 
Modern neural networks are often miscalibrated~\citep{guo2017calibration}. Post-hoc methods, including Platt scaling~\citep{niculescumizil2005predicting} and beta calibration~\citep{kull2017beta}, are commonly used to align predicted probabilities with true empirical frequencies. In our framework, probabilistic calibration is not strictly required for theoretical validity. The conformal procedure provides distribution-free safety guarantees regardless of the gate's Expected Calibration Error (ECE). However, we demonstrate empirically that correcting overconfident or underconfident gates via these methods affects threshold dynamics. Better calibration prevents the conformal procedure from selecting extreme thresholds to compensate for skewed score distributions. Recalibration may increase or decrease raw coverage depending on the initial miscalibration direction. However, it produces a more stable violation rate in finite-sample settings.

\FloatBarrier
\section{Problem Formulation}\label{sec:problem}

Let $f: \mathcal{X} \to \mathbb{R}$ be an accurate black-box model 
and $g: \mathcal{X} \to \mathbb{R}$ a simpler surrogate. Given input 
$x$ with true label $y$, we define the degradation:
\begin{equation}\label{eq:degradation}
  d(x,y) = |y - g(x)| - |y - f(x)|,
\end{equation}
which is the increase in absolute error from using $g$ instead of $f$. Positive $d$ means the surrogate is worse, and negative $d$ means it is better. Given tolerance $\tau \in \mathbb{R}$, an input is \emph{safe} if $d(x,y) \leq \tau$:
\begin{equation}
  Y(x) = \mathbf{1}[d(x,y) \leq \tau], \quad \pi = \Pr(Y = 1).
\end{equation}
A routing policy $\hat{\pi}(x) = \mathbf{1}[s(x) \geq t]$ routes to the surrogate when gate score $s(x)$ exceeds threshold $t$. The \emph{violation rate} among routed inputs is:
\begin{align}\label{eq:violation}
  V(t) &= \Pr(Y = 0 \mid s(X) \geq t) \notag\\
       &= \frac{(1-\pi)\,\mathrm{FPR}(t)}{(1-\pi)\,\mathrm{FPR}(t) 
          + \pi\,\mathrm{TPR}(t)},
\end{align}
where $\text{TPR}(t) = \Pr(s(X) \geq t \mid Y=1)$ and $\text{FPR}(t) = \Pr(s(X) \geq t \mid Y=0)$.

Note that computing degradation labels requires running both $f$ and $g$ on the training and calibration sets. This is a one-time offline cost.

\begin{definition}[Safe routing problem]
Find threshold $t$ maximizing coverage $\Pr(s(X) \geq t)$ subject to $V(t) \leq \alpha$.
\end{definition}

The parameters $\tau$ and $\alpha$ are decoupled: $\tau$ defines what ``safe'' means (continuous tolerance), while $\alpha$ controls what fraction of violations is acceptable among routed inputs. Setting $\tau = 1.0, \alpha = 0.2$ means: among inputs routed to the surrogate, at most 20\% may have degradation exceeding 1.0 units.

While we instantiate safety via degradation relative to the expensive model, the framework applies to any measurable objective that induces a binary safety label, including fidelity-based criteria or application-specific utility functions.

\FloatBarrier
\section{Method}\label{sec:method}

\subsection{Gate Training}

We train a logistic regression classifier to predict $Y$ from input features $x$. The gate outputs $s(x) = \hat{p}(Y=1|x)$, optionally calibrated via Platt scaling (also known as sigmoid calibration)~\citep{platt1999probabilistic, niculescumizil2005predicting} using a cross-validated holdout set. We use logistic regression as the gate classifier to minimize computational overhead. As a lightweight, linear baseline, it demonstrates that our conformal framework provides safe and efficient routing without requiring a complex, computationally expensive gate. Logistic regression also yields scores that are monotone in the likelihood ratio, producing concave ROC curves in the population limit~\citep{fawcett2006introduction}. This activates the tighter feasibility bound of Theorem~\ref{thm:tight_auc} and the coverage bound of Theorem~\ref{thm:coverage}.

\subsection{Conformal Threshold Selection}\label{sec:conformal}

Following the framework of Conformal Risk Control~\citep{angelopoulos2023conformal, bates2021distribution}, we use a calibration set $\{(s_i, Y_i)\}_{i=1}^n$ independent of training and test data to select the optimal routing threshold. Specifically, we select $t^*$ as the lowest threshold where the Clopper--Pearson~\citep{clopper1934use} upper confidence bound on the violation rate is at most $\alpha$:
\begin{equation}\label{eq:threshold}
  t^* = \min\Big\{ t \in \{s_1, \ldots, s_n\} : \text{UCB}_\delta(k(t), n(t)) \leq \alpha \Big\},
\end{equation}
where $n(t) = |\{i : s_i \geq t\}|$, $k(t) = |\{i : s_i \geq t, Y_i = 0\}|$, and
\begin{equation}
  \text{UCB}_\delta(k, n) = B^{-1}(1-\delta;\; k+1,\; n-k)
\end{equation}
is the exact one-sided Clopper--Pearson upper bound, with $B^{-1}$ the Beta quantile function. If no valid $t$ exists, $t^* = \infty$ (route nothing).

\begin{proposition}[Finite-sample guarantee]\label{prop:guarantee}
For any gate $s$, any distribution over $(X, Y)$, and calibration set of size $n$ exchangeable with the test point:
\[
  \Pr\!\big(V(t^*) > \alpha\big) \leq \delta.
\]
\end{proposition}

This follows directly from the coverage property of the Clopper--Pearson interval~\citep{clopper1934use}: $\Pr(\hat{p} \leq \text{UCB}_\delta(k,n)) \geq 1 - \delta$ for any true proportion $\hat{p}$, where $k \sim \text{Binomial}(n, \hat{p})$. With $n \approx 300$ calibration points (our setup), the bound is tight: for $k=0$ unsafe among $n=300$ routed calibration points, $\text{UCB}_{0.10}(0, 300) = 0.0076$, enabling routing with $\alpha$ as low as 0.01.

\begin{algorithm}[H]
\caption{Conformal Gate Routing}\label{alg:routing}
\begin{algorithmic}[1]
\REQUIRE Calibration scores $\{s_i\}_{i=1}^n$, labels $\{Y_i\}_{i=1}^n$, risk $\alpha$, confidence $\delta$
\STATE Sort unique scores: $t_1 < t_2 < \ldots < t_m$
\FOR{$j = 1, \ldots, m$}
  \STATE $n_j \gets |\{i : s_i \geq t_j\}|$, \quad $k_j \gets |\{i : s_i \geq t_j, Y_i = 0\}|$
  \IF{$\text{UCB}_\delta(k_j, n_j) \leq \alpha$}
    \RETURN $t^* = t_j$
  \ENDIF
\ENDFOR
\RETURN $t^* = \infty$ \COMMENT{Abstain from routing}
\end{algorithmic}
\end{algorithm}

\begin{remark}[Validity under threshold search]\label{rem:adaptivity}
Algorithm~\ref{alg:routing} evaluates multiple candidate thresholds on the same calibration data. We now argue that post-selection validity holds without any multiplicity correction. Condition on the realized gate scores $\{s_i\}_{i=1}^n$. The threshold ordering $t_1 < t_2 < \cdots < t_m$ is then deterministic, and the routed sets $\{i : s_i \geq t_j\}$ are fixed. The safety labels $\{Y_i\}_{i=1}^n$ depend on $(x_i, y_i)$ and the pre-trained models $f$ and $g$. Because the gate is trained on a separate split, these labels remain exchangeable conditional on the scores. Each Clopper--Pearson test is therefore exact given the scores. Algorithm~\ref{alg:routing} tests thresholds in a fixed ascending order and accepts the first one satisfying $\mathrm{UCB}_\delta(k_j, n_j) \leq \alpha$. By the fixed-sequence testing principle~\citep{marcus1976closed}, this controls the family-wise error rate at level $\delta$. The unconditional guarantee $\Pr(V(t^*) > \alpha) \leq \delta$ follows by iterated expectation. No Bonferroni correction or alternative calibration framework is required.
\end{remark}

\FloatBarrier
\section{Theoretical Analysis}\label{sec:theory}

\subsection{Routing Feasibility}

\begin{proposition}[Feasibility condition]\label{prop:feasibility}
The constraint $V(t) \leq \alpha$ is satisfiable with positive coverage if and only if there exists threshold $t$ such that:
\begin{equation}\label{eq:neyman}
  \frac{\mathrm{TPR}(t)}{\mathrm{FPR}(t)} \geq C(\pi, \alpha) \triangleq \frac{(1-\pi)(1-\alpha)}{\pi\alpha}.
\end{equation}
\end{proposition}

The full algebraic derivation appears in Appendix~A.

\paragraph{Remark on Optimality.} 
Condition~\eqref{eq:neyman} implies that for any fixed false positive rate, the violation rate is minimized by maximizing the True Positive Rate. By the Neyman--Pearson lemma~\citep{neyman1933problem}, the likelihood ratio test is the most powerful test for this objective. Thus, the optimal gate $s(x)$ is the likelihood ratio $P(Y=1|x)/P(Y=0|x)$, which monotonically maps to the posterior $P(Y=1|x)$ used in our experiments.

The critical ratio $C(\pi, \alpha)$ captures the interaction between $\tau$ (via $\pi$) and the risk budget $\alpha$. When $\pi$ is large (lenient $\tau$, most inputs safe), $C$ is small and even weak gates suffice. When $\pi$ is small (strict $\tau$), $C$ grows and only gates with strong separation can route safely.

Following standard geometric formulations of Receiver Operating Characteristic (ROC) analysis~\citep{fawcett2006introduction}, we can translate this likelihood-ratio condition directly into ROC space.

\begin{corollary}[ROC Feasibility Geometry]
Let 
\[
C(\pi,\alpha) = \frac{(1-\pi)(1-\alpha)}{\pi \alpha}.
\]
Feasible routing exists if and only if the ROC curve
intersects the half-space
\[
\{(u,v)\in[0,1]^2 : v \ge C(\pi,\alpha)\,u\}.
\]
Equivalently, there exists a threshold $t$ such that
\(
\mathrm{TPR}(t) \ge C(\pi,\alpha)\,\mathrm{FPR}(t).
\)
\end{corollary}

The full algebraic derivation appears in Appendix~A.
\FloatBarrier
\subsection{Critical AUC}

\begin{theorem}[Sufficient AUC for feasibility]\label{thm:critical_auc}
Define the critical AUC:
\begin{equation}\label{eq:critical_auc}
  \Phi_c(\pi, \alpha) = \min\!\Big(1,\; \tfrac{C(\pi,\alpha)}{2}\Big).
\end{equation}
If\/ $\mathrm{AUC} \geq \Phi_c(\pi, \alpha)$, then there exists a routing threshold $t$ satisfying $V(t) \leq \alpha$ with positive coverage.
\end{theorem}

The full algebraic derivation appears in Appendix~A.

\begin{remark}
Theorem~\ref{thm:critical_auc} is sufficient but conservative.
Theorem~\ref{thm:tight_auc} partially closes the gap: 
at $\tau=0.5$, the required AUC drops from 
$\Phi_c = 1.0$ to $\Phi_c^* = 0.76$.
The residual gap (observed AUC $0.57 < 0.76$, yet routing succeeds) 
reflects a fundamental limit of scalar summaries: feasibility 
is a local ROC property (Corollary~1), while AUC is global. 
No single-scalar condition can be both necessary and sufficient.
This yields a hierarchy: Corollary~1 (exact, requires full ROC) 
$\to$ Theorem~\ref{thm:tight_auc} (tight scalar, concave ROC) 
$\to$ Theorem~\ref{thm:critical_auc} (sufficient, any ROC).
\end{remark}

\begin{corollary}
As $\alpha \to 0$ or $\pi \to 0$: $\Phi_c \to 1$, so near-perfect ranking is required.
For $\pi = 0.82$ (our $\tau = 2.0$) and $\alpha = 0.2$, we have $C = 0.88$ and $\Phi_c = 0.44$. Our observed AUC of $0.59$ exceeds this threshold, confirming feasibility.
For $\pi = 0.66$ ($\tau = 0.5$) and $\alpha = 0.2$: 
$\Phi_c = 1.0$ but $\Phi_c^* = 0.76$ (Theorem~\ref{thm:tight_auc}), 
yet routing succeeds empirically with 24\% coverage, 
illustrating that even the tight bound cannot capture local ROC structure.
\end{corollary}

\begin{figure}[htbp]
  \centering
  \includegraphics[width=\columnwidth]{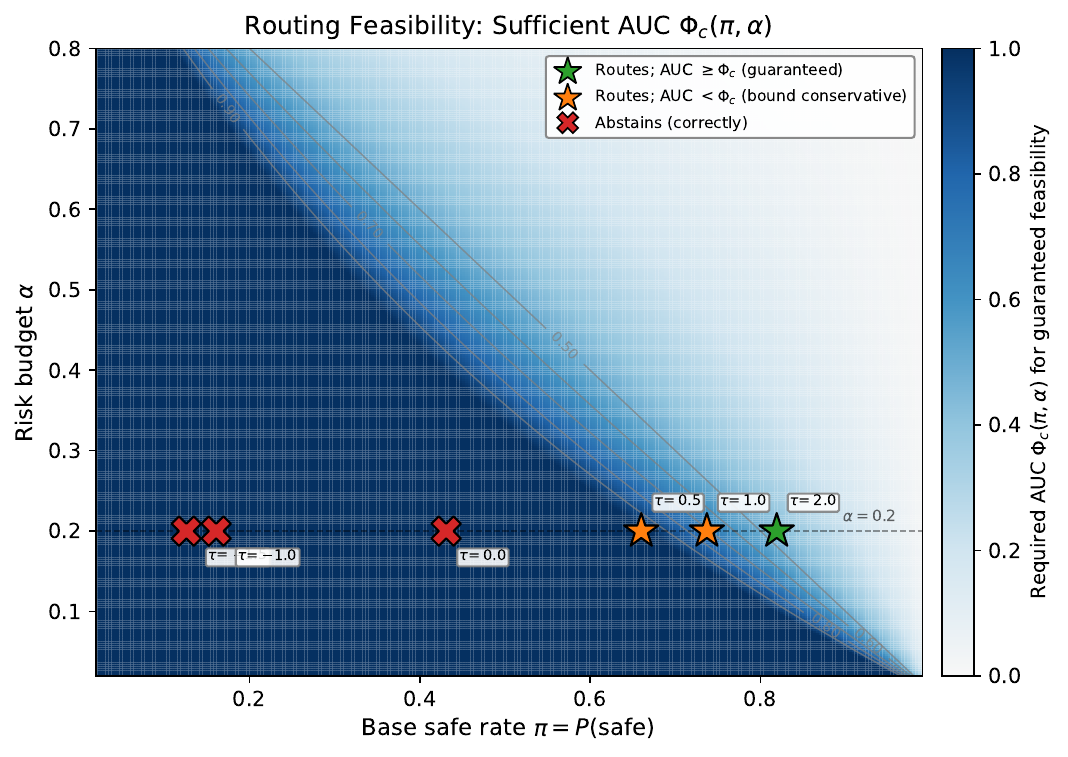}
  \caption{Feasibility landscape: critical AUC $\Phi_c(\pi, \alpha)$ as a function of base safe rate $\pi$ and risk budget $\alpha$.
  Darker regions require higher AUC for guaranteed feasibility.
  Stars mark our six $\tau$ operating points at $\alpha = 0.2$, colored by empirical outcome: green if routing achieves positive coverage, red if coverage is zero.
  At $\tau \leq 0$ (low $\pi$), the gate falls in the hard region and correctly abstains. At $\tau = 0.5$ and $\tau = 1.0$, routing succeeds despite AUC falling below $\Phi_c$, illustrating that the sufficient condition is conservative.
  At $\tau = 2.0$, AUC exceeds $\Phi_c$ and feasibility is guaranteed.}\label{fig:feasibility}
\end{figure}

\begin{theorem}[Tight AUC under concavity]\label{thm:tight_auc}
Let $C = C(\pi,\alpha)$. For any gate with \emph{concave} ROC curve, define
\begin{equation}\label{eq:tight_auc}
  \Phi_c^*(\pi,\alpha) = 
  \begin{cases}
    1 - \dfrac{1}{2C} & \text{if } C > 1,\\[6pt]
    \tfrac{1}{2}       & \text{if } C \leq 1.
  \end{cases}
\end{equation}
\begin{enumerate}[leftmargin=*,topsep=2pt,itemsep=1pt]
  \item[\emph{(i)}] If\/ $\mathrm{AUC} \geq \Phi_c^*$, feasible routing exists.
  \item[\emph{(ii)}] The bound is tight: for any $\varepsilon > 0$, a concave ROC with $\mathrm{AUC} = \Phi_c^* - \varepsilon$ exists that is infeasible.
  \item[\emph{(iii)}] $\Phi_c^* < \Phi_c$ for all $C > 1$, strictly improving Theorem~\ref{thm:critical_auc}.
\end{enumerate}
\end{theorem}

The full algebraic derivation appears in Appendix~A.

\begin{theorem}[Coverage bound]\label{thm:coverage}
Under the conditions of Theorem~\ref{thm:tight_auc} with $C>1$ and
$\mathrm{AUC}=A \ge \Phi_c^*$, the optimal coverage subject to
$V(t)\le \alpha$ satisfies
$\mathrm{Cov}^* \ge \min\!\big\{1,\; (2A-1)(1-\pi) / [(C-1)\alpha]\big\}$.
The bound is tight (Appendix~\ref{app:coverage_proof}).
\end{theorem}

\FloatBarrier
\subsection{The Calibration--Efficiency Separation}

The conformal guarantee (Proposition~\ref{prop:guarantee}) holds for any gate score $s(x)$, independent of its probabilistic calibration. However, calibration affects efficiency (coverage) and the finite-sample stability of the selected threshold.

For a perfectly calibrated gate with $s(x)=\Pr(Y=1\mid x)$, the expected violation rate at threshold $t$ equals $1-t$, so the conformal procedure selects $t^*\approx 1-\alpha$ (up to finite-sample correction). If the gate is overconfident (e.g., $s(x)=0.9$ when the true probability is $0.7$), conformal compensates by shifting to a more extreme threshold (e.g., $t^*=0.95$). Underconfidence similarly forces lower thresholds. In all cases safety is preserved, but extreme thresholds increase sensitivity to finite-sample noise, leading to higher variance in realized violations.

Thus, calibration determines which inputs are routed, how many are routed, and threshold stability, but not the validity of the risk bound. Appendix~E confirms this empirically: beta calibration~\citep{kull2017beta}, temperature scaling~\citep{guo2017calibration}, and isotonic regression~\citep{zadrozny2002transforming} reduce ECE (0.18 to 0.06) while preserving safety. Coverage changes only marginally ($\le 2.5$pp). Overall efficiency remains primarily governed by ranking quality (AUC), while calibration stabilizes finite-sample thresholding.

\FloatBarrier
\section{Experiments}\label{sec:experiments}

\subsection{Setup}

We evaluate our approach on 35 regression datasets from LCBench~\citep{zimmer2021auto,vanschoren2014openml}, where each dataset corresponds to a different OpenML classification task and contains 2000 hyperparameter configurations with recorded validation accuracy at epoch 50. Inputs are hyperparameter settings (learning rate, batch size, architecture choices, etc.) and the target is validation accuracy. Data is split 55/15/15/15 into train, validation, calibration, and test sets with strict separation. Additional experiments on 22 general-purpose OpenML regression datasets (air quality, housing, energy, chemistry, physics, robotics, etc.) are reported in Appendix~\ref{app:openml-general}.

The expensive model $f$ is a random forest~\citep{breiman2001random}, treated as a black-box ensemble with 1500 trees. The cheap surrogate $g$ is a decision tree with depth selected on the validation set from $\{2,3,4,5,7,9,11,13,15\}$ by minimizing MAE. Degradation is $d(x) = |y - g(x)| - |y - f(x)|$.

We evaluate six tolerance levels $\tau \in \{-1.5, -1.0, 0.0, 0.5, 1.0, 2.0\}$ spanning strict ($\tau < 0$: surrogate must be better) to lenient ($\tau = 2.0$: up to 2.0 extra error acceptable). We test ten risk levels $\alpha \in \{0.05, 0.10, \ldots, 0.80\}$ with $\delta = 0.10$.

\subsection{Methods}
We compare against five baselines:
\begin{enumerate}
  \item \textbf{Gate conformal} (ours): logistic regression gate with Clopper--Pearson threshold selection.
  \item \textbf{Naive ($t=0.5$)}: route if gate score $\geq 0.5$. No formal guarantee.
  \item \textbf{Oracle}: route iff $d(x) \leq \tau$. This provides an upper bound on achievable coverage.
  \item \textbf{Always-BB / Always-CM}: never/always route.
  \item \textbf{Random (matched)}: route uniformly at random at conformal's coverage level.
\item \textbf{Regression Conformal (Baseline):}
As an alternative to the binary gate, we predict the degradation $d(x)$ via regression. We train a ridge model~\citep{hoerl1970ridge} $\hat{h}(x)$ and apply split conformal regression~\citep{papadopoulos2002inductive, lei2018distribution} to obtain an upper bound $\hat{d}_{\mathrm{UCB}}(x)$ at level $1-\alpha$, routing when $\hat{d}_{\mathrm{UCB}}(x)\le\tau$. A non-linear \texttt{HistGradientBoostingRegressor} yields similar results (Appendix~\ref{app:hgbr}), indicating the gap is not driven by regressor capacity.
\end{enumerate}

\FloatBarrier
\subsection{Gate Properties, Feasibility, and Guarantee Reliability}
\label{sec:gate_and_reliability}

\begin{table}[htbp]
\caption{Gate properties across tolerance levels. $\pi$: base safe rate; AUC: gate ranking quality (median [IQR]); ECE: calibration error. $\Phi_c$/$\Phi_c^*$: sufficient AUC thresholds from Theorems~\ref{thm:critical_auc}/\ref{thm:tight_auc}. Coverage and mean violation at $\alpha=0.2$ averaged over 35 datasets.}\label{tab:gate}
\centering\small
\begin{tabular}{@{}rcccccccc@{}}
\toprule
$\tau$ & $\pi$ & AUC & ECE & $\Phi_c$ & $\Phi_c^*$ & \multicolumn{2}{c}{Conf.\ ($\alpha\!=\!0.2$)} \\
\cmidrule(lr){7-8}
& & (med.) & & & & Cov. & Viol.  \\
\midrule
$-1.5$ & 0.13 & 0.71 & 0.07 & 1.00 & 0.92 & 0.000 & --- \\
$-1.0$ & 0.16 & 0.70 & 0.09 & 1.00 & 0.95 & 0.000 & --- \\
$0.0$  & 0.43 & 0.54 & 0.25 & 1.00 & 0.88 & 0.003 & 0.03 \\
$0.5$  & 0.66 & 0.57 & 0.20 & 1.00 & 0.76 & 0.243 & 0.15 \\
$1.0$  & 0.74 & 0.58 & 0.17 & 0.70 & 0.65 & 0.446 & 0.14 \\
$2.0$  & 0.82 & 0.59 & 0.12 & 0.44 & 0.50 & 0.661 & 0.12 \\
\bottomrule
\end{tabular}
\end{table}

\begin{table}[htbp]
\caption{Fraction of (dataset, $\tau$) pairs where the realized violation rate exceeds target $\alpha$. Gate conformal is broadly consistent with $\delta = 0.10$. Baselines violate substantially more often.}\label{tab:guarantee}
\centering\small
\begin{tabular}{@{}lccc@{}}
\toprule
$\alpha$ & Gate Conf. & Reg.\ Conf. & Naive \\
\midrule
0.10 & 7/34 \hphantom{0}(21\%) & 31/50 \hphantom{0}(62\%) & 151/172 (88\%) \\
0.20 & 8/66 \hphantom{0}(12\%) & 57/102 (56\%) & 128/172 (74\%) \\
0.30 & 6/82 \hphantom{00}(7\%) & 80/144 (56\%) & 106/172 (62\%) \\
0.50 & 8/112 \hphantom{0}(7\%) & 64/181 (35\%) & \hphantom{0}62/172 (36\%) \\
\bottomrule
\end{tabular}
\end{table}

Tables~\ref{tab:gate}--\ref{tab:guarantee} summarize (i) when routing is feasible (Proposition~\ref{prop:feasibility}) and (ii) how often the target is met across datasets. 

\paragraph{$\pi$ drives feasibility (ROC geometry).}
At $\tau \le 0$, $\pi < 0.45$ and the sufficient bound gives $\Phi_c=1.0$, so the conformal procedure abstains (zero coverage), consistent with Proposition~\ref{prop:feasibility}. At $\tau \ge 1.0$, $\pi$ increases ($\pi \ge 0.74$), reducing the required separation (smaller $C(\pi,\alpha)$) and enabling substantial coverage (e.g., 66\% at $\tau=2.0$).

\paragraph{Moderate AUC can suffice due to local ROC slope.}
Although Theorems~\ref{thm:critical_auc}--\ref{thm:tight_auc} provide useful scalar feasibility thresholds, feasibility is ultimately a local ROC condition: the conformal search can select a threshold that isolates a small, high-purity region (steep initial ROC slope) even when global AUC is modest. This explains successful routing at $\tau\in\{0.5,1.0\}$ despite AUC falling below $\Phi_c$.

\paragraph{ECE affects efficiency, not validity.}
ECE varies widely (0.07--0.25) and peaks near $\tau=0$ where the task is most ambiguous, but the distribution-free guarantee depends only on the calibration procedure, not on probabilistic calibration quality (Section~5.3). Miscalibration primarily shifts the selected threshold $t^*$ and thus coverage, rather than invalidating the risk control.

\paragraph{Guarantee reliability across datasets.}
Table~\ref{tab:guarantee} is the main empirical validity check. Gate conformal exceeds $\alpha$ on 7--12\% of (dataset, $\tau$) pairs for $\alpha \geq 0.2$. This is broadly consistent with $\delta = 0.10$. At $\alpha = 0.10$, the exceedance rate rises to 21\%. This occurs because at strict $\alpha$, the conformal procedure selects high thresholds that route very few calibration points. When the effective routed calibration size $n(t^*)$ is small, the CP bound is correct but necessarily loose. The denominator of (dataset, $\tau$) pairs achieving positive coverage is also dominated by borderline-feasible settings at this strict level. For $\alpha \geq 0.20$, routed calibration sets are larger and exceedance rates (7--12\%) align with the nominal $\delta = 0.10$. In contrast, regression conformal and naive thresholding violate $\alpha$ on 35--88\% of cases. This indicates that direct regression on degradation is substantially less reliable for conditional risk control.

\paragraph{Robustness to model families.}
We replicate the evaluation with XGBoost~\citep{chen2016xgboost} and MLP~\citep{goodfellow2016deep} as the expensive model (same decision-tree surrogate). Qualitative conclusions remain unchanged. Full results appear in Appendices~D--E.

\FloatBarrier
\subsection{Coverage--Violation Pareto Frontier}

Figure~\ref{fig:pareto} shows the empirical Pareto frontier characterizing the standard risk-coverage tradeoff~\citep{elyaniv2010foundations} at $\alpha = 0.2$. Gate conformal traces a frontier consistently below the $\alpha$ line, reaching 66\% coverage at $\tau = 2.0$ with only 12\% mean violation. Regression conformal achieves lower coverage (58\%) at higher violation (15\%), and naive thresholding violates $\alpha$ at every operating point. Gate conformal dominates on both axes at all practical $\tau$ values ($\geq 0.5$).

At low $\tau$, regression conformal routes a tiny fraction (0.1--3\%) with massive violation (8--51\%), while gate conformal correctly abstains. This illustrates the value of the conformal abstention mechanism: routing nothing is preferable to routing with violated guarantees.

\begin{figure}[htbp]
  \centering
  \includegraphics[width=\columnwidth]{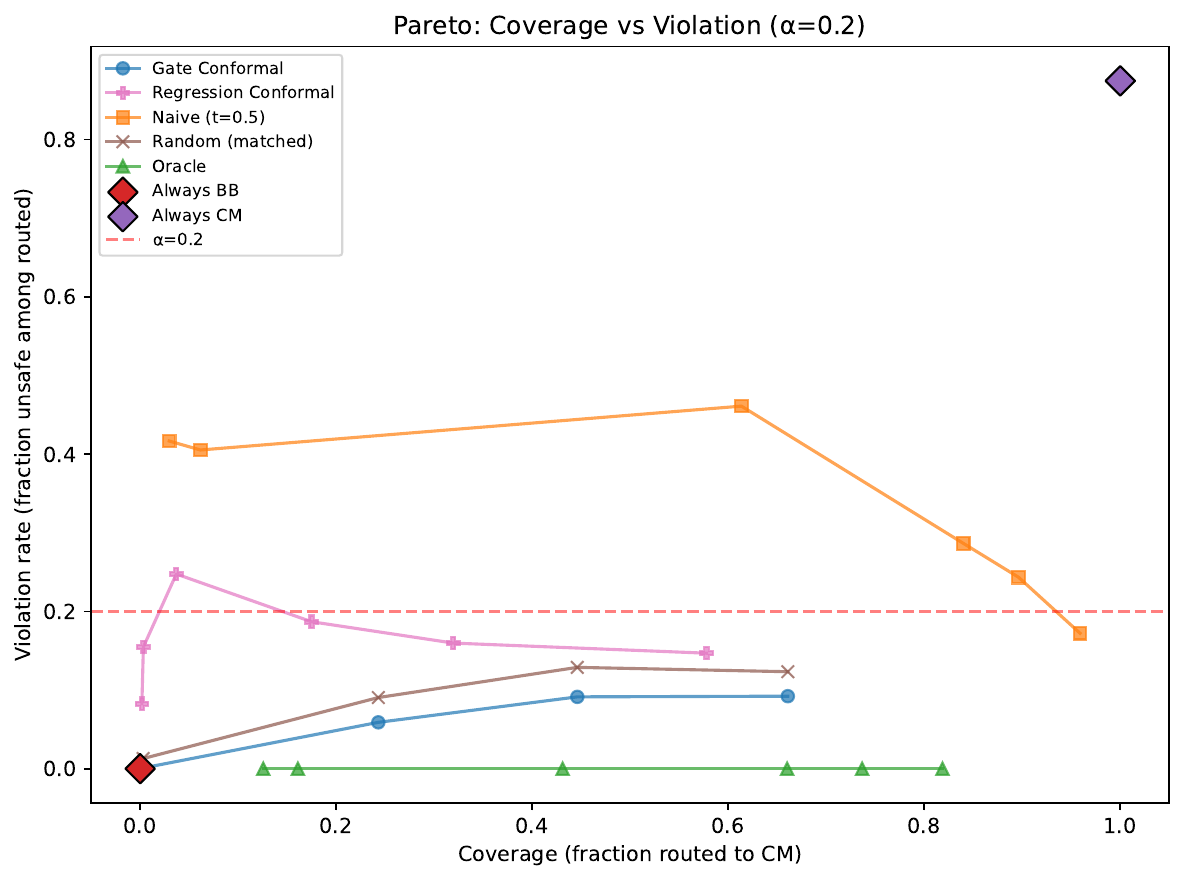}
  \caption{Coverage vs.\ violation rate at $\alpha = 0.2$. Each point corresponds to a $\tau$ value averaged over 35 datasets. Gate conformal remains below $\alpha$. Regression conformal and naive operate above it.}\label{fig:pareto}
\end{figure}

\subsection{Calibration--Efficiency Separation}

\begin{figure}[htbp]
  \centering
  \includegraphics[width=\columnwidth,height=0.45\textheight,keepaspectratio]{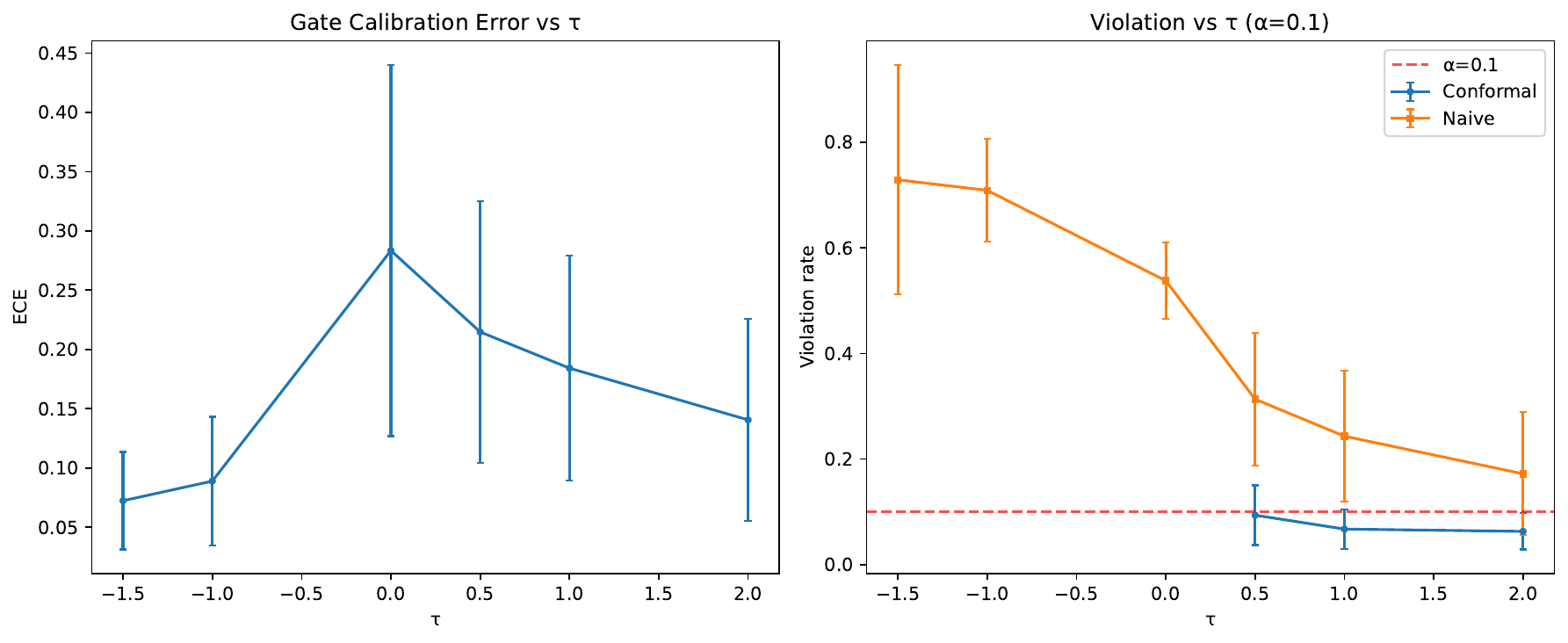}
  \caption{Left: gate ECE varies from 0.07 to 0.28 across $\tau$. Right: despite high ECE, conformal violation remains below $\alpha = 0.1$. Naive violations are 4--8$\times$ higher.}\label{fig:ece}
\end{figure}

Figure~\ref{fig:ece} demonstrates the calibration--efficiency separation empirically. The gate's ECE peaks at $\tau = 0$ (0.28) where class balance is near 50/50 and the gate's probability estimates are least reliable. Yet, conformal violation remains well below $\alpha = 0.1$ at all $\tau$. The primary cost of this initial miscalibration is threshold instability rather than strict invalidity. At $\tau = 0$, coverage is only 0.3\% despite 43\% of inputs being theoretically safe ($\pi = 0.43$).

As detailed in Appendix~\ref{app:recal}, applying post-hoc recalibration substantially reduces ECE but does not uniformly increase coverage. Instead, by preventing the conformal procedure from selecting extreme thresholds to compensate for skewed scores, recalibration stabilizes the threshold in a denser region of the data. This occasionally results in slightly lower coverage but yields a realized violation rate that is more stable and truer to the underlying data distribution. Ultimately, this indicates that while proper calibration stabilizes thresholds, the remaining efficiency gap is primarily due to ranking limitations (AUC $= 0.58$), not miscalibration.

\FloatBarrier
\subsection{Gate Conformal vs.\ Regression Conformal}

We next compare against the natural regression-based alternative: predict degradation
\(\,d(x)\,\) from input features and apply split conformal intervals, routing when the
upper bound lies below \(\tau\) (Section~\ref{sec:experiments}).
This baseline has no direct guarantee on conditional violation among routed inputs,
so we evaluate both coverage and violations of \(\alpha\) per dataset.

\begin{table}[H]
\caption{Per-dataset comparison at practical operating points. Gate conformal achieves higher coverage on more datasets while violating $\alpha$ on far fewer.}\label{tab:headtohead}

\centering\small
\begin{tabular}{@{}llccccc@{}}
\toprule
$\tau$ & $\alpha$ & \#DS & \multicolumn{2}{c}{Higher Cov.} & \multicolumn{2}{c}{Violates $\alpha$} \\
\cmidrule(lr){4-5}\cmidrule(lr){6-7}
& & & Gate & Reg & Gate & Reg \\
\midrule
0.5 & 0.2 & 23 & 10 & 11 & 2 & 15 \\
0.5 & 0.3 & 32 & 16 & 14 & 2 & 18 \\
1.0 & 0.2 & 29 & 18 & 7 & 4 & 11 \\
1.0 & 0.3 & 32 & 20 & 7 & 1 & 12 \\
2.0 & 0.2 & 32 & 18 & 6 & 2 & 7 \\
2.0 & 0.3 & 35 & 13 & 9 & 2 & 7 \\
\bottomrule
\end{tabular}
\vspace{-0.5cm}
\end{table}

Gate conformal achieves higher coverage on more datasets while violating $\alpha$
far less often than regression conformal (Table~\ref{tab:headtohead}).
At $\tau=1.0, \alpha=0.2$, the mean coverage gain is $+23.9\%$ 
($p < 0.001$, Wilcoxon signed-rank test~\citep{demsar2006statistical}).

\begin{remark}[Marginal vs.\ conditional risk control in routing]\label{rem:marginal}
Split conformal regression guarantees $\Pr(d(X) \leq \hat{d}(X) + \hat{q}) \geq 1 - \alpha$ marginally over all test inputs. Routing, however, conditions on the event $\{\hat{d}(X) + \hat{q} \leq \tau\}$, which selects inputs where predicted degradation is low. This conditioning can concentrate the routed subset on inputs where the regression model is overconfident. The resulting conditional violation rates can far exceed $\alpha$. This failure is structural, not a limitation of regressor capacity. Replacing ridge regression with a nonlinear \texttt{HistGradientBoostingRegressor} increases the violation rate from 56\% to 72\% of settings at $\alpha = 0.2$ (Table~\ref{tab:guarantee_full}). A more accurate regressor routes more aggressively into the biased subset, making the problem worse. The classification formulation avoids this issue by directly targeting the conditional violation rate on the routed set.
\end{remark}

\FloatBarrier
\subsection{Guarantee Reliability Across $\alpha$}

\begin{figure}[htbp]
  \centering
  \includegraphics[width=\columnwidth]{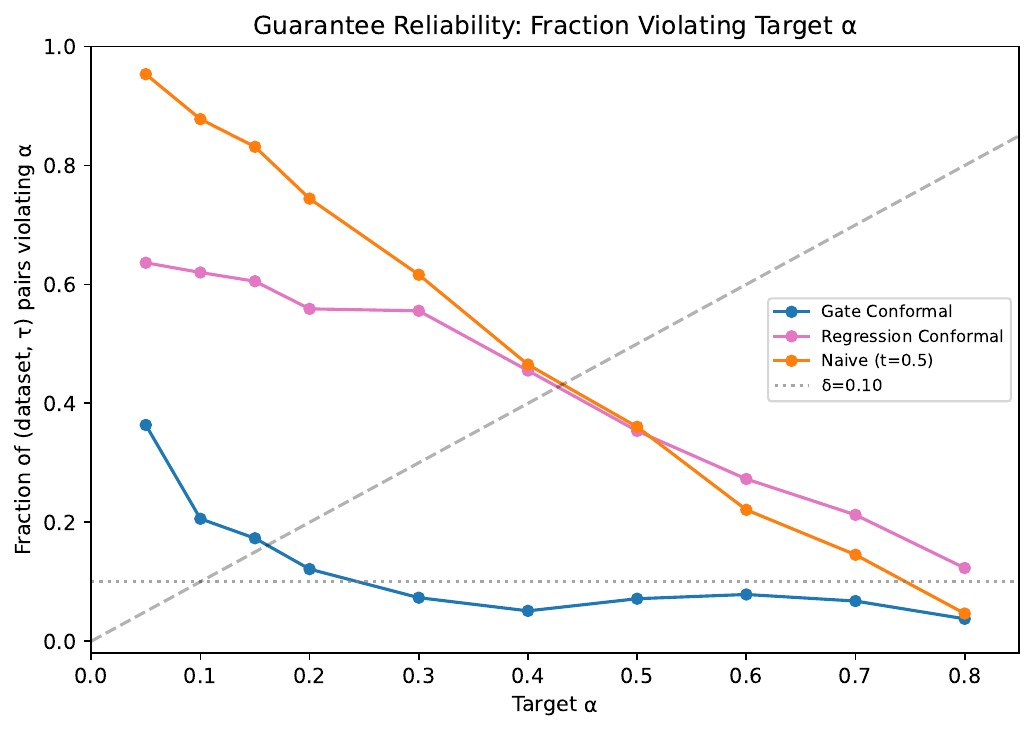}
  \caption{Fraction of (dataset, $\tau$) pairs violating target $\alpha$ across the full $\alpha$ range. Gate conformal remains near the $\delta = 0.10$ line. Baselines converge only at very permissive $\alpha$.}\label{fig:fraction}
\end{figure}

Figure~\ref{fig:fraction} shows guarantee reliability across the full $\alpha$ range. Gate conformal maintains violation rates near or below $\delta = 0.10$ for $\alpha \geq 0.15$, confirming the theoretical guarantee. At very strict $\alpha = 0.05$, the rate rises to 36\%. This occurs because few (dataset, $\tau$) pairs have enough safe points in the calibration set to verify at $\alpha = 0.05$ with $\delta = 0.10$, so the denominator consists primarily of borderline cases. Regression conformal violates at 55--63\% across most $\alpha$ values, and naive thresholding exceeds 60\% at $\alpha \leq 0.3$.

We analyze calibration set size requirements in detail in Appendix~\ref{app:calsize}, 
showing that $n \approx 300$ is comfortably in the regime where the bound is tight.

\FloatBarrier
\section{Discussion}\label{sec:discussion}

\paragraph{Practical implications.}
Our results suggest the following workflow for safe model routing. 
(1) Train a gate classifier. 
(2) Check feasibility via Proposition~\ref{prop:feasibility} using estimated $\pi$ and desired $\alpha$. 
(3) If feasible, run Algorithm~\ref{alg:routing} on a calibration set sized appropriately for the target risk level, for example ensuring $n \gg 1/\alpha$ to obtain non-vacuous bounds. 
(4) Deploy with formal guarantees.

Post-hoc recalibration, such as temperature scaling or beta calibration, is not required for theoretical validity but is recommended in practice. Maximum achievable coverage is primarily limited by ranking quality (AUC) rather than miscalibration. Recalibration does not substantially increase coverage, but by correcting over- or underconfident scores it prevents the conformal procedure from selecting extreme thresholds. This stabilizes finite-sample behavior and yields violation rates more closely aligned with the target $\alpha$ (Appendix~\ref{app:recal}).

The framework requires only a gate trained on input features and a calibration set with binary safety labels. These requirements are lightweight and model-agnostic, and extend naturally to domains such as LLM routing using human evaluation or automated scoring.

\paragraph{Computational cost.}
The gate is a logistic regression requiring a single dot product per input, or $O(d)$ operations where $d$ is the feature dimension. The expensive model is a random forest with 1500 trees, requiring $O(1500 \times \text{depth})$ operations per input. This is approximately three orders of magnitude more expensive than the gate. At 66\% coverage ($\tau = 2.0$), proactive routing avoids running the expensive model on two-thirds of inputs. The gate overhead is negligible compared to either model. Unlike cascading, which must run the surrogate on every input before deciding to escalate, proactive routing runs only the gate on all inputs.

\paragraph{When routing fails.}
At $\tau = -1.0$, only 16\% of inputs are safe ($\pi = 0.16$).  
With $\alpha = 0.2$, Proposition~\ref{prop:feasibility} requires
\[
  \frac{\mathrm{TPR}(t)}{\mathrm{FPR}(t)} \ge C(0.16,0.2)
  = \frac{0.84 \times 0.8}{0.16 \times 0.2}
  = 21.
\]
No empirical ROC point satisfies this, even with AUC $0.70$. The sufficient bound gives $\Phi_c = 1.0$, implying that near-perfect ranking would be required.

Accordingly, conformal returns $t^*=\infty$ and abstains.  
Regression conformal instead routes 0.3\% of inputs with 15.5\% mean violation, exceeding $\alpha$ on most datasets.

\paragraph{Classification vs.\ regression for routing.}
The dominance of classification over regression in our setting has a clear explanation: predicting whether $d(x) \leq \tau$ (a binary question) requires only finding a decision boundary, while predicting $d(x)$ (a continuous quantity) requires accurate function approximation everywhere. When input features are weakly predictive of degradation (as in our tabular setting), the classification formulation degrades gracefully (lower coverage) while the regression formulation leads to frequent violations.

\paragraph{Limitations.}
Several limitations bound our conclusions and suggest directions for future work.

\emph{Marginal vs.\ conditional guarantees.}
Our guarantee controls the average violation rate over random draws of the calibration set, not the violation rate conditional on specific input subregions.
Exact distribution-free conditional coverage is known to be impossible without structural assumptions~\citep{angelopoulos2023conformal}.
The routed set could therefore contain subgroups with higher-than-$\alpha$ violation rates even while the overall rate is controlled.
Extending to approximate conditional guarantees via group-conditional conformal methods is an important direction.

\emph{Tolerance selection.}
We report results across multiple degradation tolerances $\tau$ for analysis, running the conformal procedure independently for each $(\tau,\alpha)$ pair. 
In deployment, $\tau$ represents an application-dependent tolerance. 
If a practitioner wishes to select $\tau$ adaptively using calibration data, a separate tuning split or a Learn-Then-Test style procedure \citep{angelopoulos2021learn} can be used to preserve the nominal $\delta$ guarantee.

\emph{Exchangeability and distribution shift.}
The guarantee assumes calibration and test data are exchangeable.
Under covariate shift, where $P_{\text{test}}(X) \neq P_{\text{cal}}(X)$, the guarantee formally breaks.
The Clopper--Pearson bound is exact (finite-sample), so mild distribution shift may not immediately produce empirical violations, but the formal guarantee no longer applies.
Weighted conformal prediction~\citep{tibshirani2019conformal}, which reweights calibration points by likelihood ratios $dP_{\text{test}}/dP_{\text{cal}}$, could extend the framework to known shift. However, estimating these ratios introduces additional uncertainty.

Empirically, we observe that even under assumed exchangeability, finite-sample variability across different random splits can lead to realized violation rates slightly exceeding $\delta$ for strict $\alpha$ targets.

\emph{Input-space separability.}
Coverage is fundamentally bounded by how well input features $x$ separate safe from unsafe regions.
If degradation depends primarily on label noise or model stochasticity rather than input features, no gate can achieve high AUC, and coverage will be low regardless of calibration or conformal procedure.
This is a limitation of the routing problem itself, not our method: some model pairs on some data simply do not admit efficient input-based routing.

\section{Conclusion}\label{sec:conclusion}

We introduced proactive routing to interpretable surrogates with 
exact, distribution-free safety guarantees. A classification gate 
predicts input-level safety, and Clopper--Pearson conformal 
calibration selects a routing threshold that controls the violation 
rate on the routed subset with probability $1-\delta$. Our analysis 
characterizes routing feasibility through a likelihood-ratio condition 
and a sufficient critical AUC bound, clarifying when safe routing is 
achievable. Across 35 datasets and multiple model families, gate 
conformal routing achieves higher coverage with controlled violations 
compared to regression-based baselines. Future work includes 
extensions to hierarchical routing frameworks and robustness under 
distribution shift.

\bibliography{references}
\bibliographystyle{plainnat}

\clearpage
\appendix
\section{Proofs}\label{app:proofs}

\subsection{Proof of Proposition~\ref{prop:guarantee}}

The guarantee follows from the coverage property of Clopper--Pearson intervals combined with the fixed-sequence testing argument of Remark~\ref{rem:adaptivity}. For any fixed threshold $t$, let $p^* = V(t)$ denote the true violation rate. On the calibration set, $k(t) \sim \text{Binomial}(n(t), p^*)$. By the Clopper--Pearson construction~\citep{clopper1934use}:
\[
  \Pr\big(p^* \leq \text{UCB}_\delta(k(t), n(t))\big) \geq 1 - \delta.
\]
Algorithm~\ref{alg:routing} selects $t^*$ adaptively from $\{t_1, \ldots, t_m\}$. Conditioning on the gate scores $\{s_i\}_{i=1}^n$, the threshold ordering and routed index sets become deterministic. The labels $\{Y_i\}$ remain exchangeable conditional on the scores because the gate is trained on an independent split. Each conditional CP test is therefore exact. The fixed-sequence testing principle~\citep{marcus1976closed} ensures that accepting the first threshold satisfying $\text{UCB}_\delta \leq \alpha$ in a pre-specified order controls the family-wise error rate at level $\delta$. By iterated expectation, $\Pr(V(t^*) > \alpha) \leq \delta$ holds unconditionally.

This argument requires that the calibration set is independent of the gate training data and that calibration points are exchangeable with test points. Our 4-way split ensures both conditions. The Clopper--Pearson bound is exact and not asymptotic, so the guarantee holds for any sample size $n$.

\subsection{Proof of Theorem~\ref{thm:critical_auc}}

The ROC curve $\{(\text{FPR}(t), \text{TPR}(t)) : t \in \mathbb{R}\}$ satisfies $\text{AUC} = \int_0^1 \text{TPR}(u)\,du$ where $u = \text{FPR}$.

Suppose no point on the ROC satisfies~\eqref{eq:neyman}, i.e., $\text{TPR}(u) < C(\pi,\alpha) \cdot u$ for all $u \in (0, 1]$. Then:
\[
  \text{AUC} = \int_0^1 \text{TPR}(u)\,du < \int_0^1 C \cdot u\,du = \frac{C}{2}.
\]
Taking the contrapositive: if $\text{AUC} \geq C/2 = \Phi_c(\pi,\alpha)$, then there exists a point satisfying~\eqref{eq:neyman}. Since AUC $\leq 1$, we clip at 1.

Note this is a necessary condition for infeasibility, not sufficient for feasibility: $\text{AUC} \geq \Phi_c$ guarantees some point on the ROC has high enough likelihood ratio, but a non-concave ROC could still have that point at a degenerate threshold.

\subsection{Proof of Proposition~\ref{prop:feasibility}}

\begin{proof}
By Bayes' rule:
$V(t) = \frac{(1-\pi)\mathrm{FPR}(t)}{(1-\pi)\mathrm{FPR}(t) + \pi\,\mathrm{TPR}(t)}$.
Setting $V(t) \leq \alpha$:
$(1-\pi)\mathrm{FPR}(t) \leq \alpha[(1-\pi)\mathrm{FPR}(t) + \pi\,\mathrm{TPR}(t)]$,
which rearranges to $(1-\alpha)(1-\pi)\mathrm{FPR}(t) \leq \alpha\pi\,\mathrm{TPR}(t)$.
Dividing both sides by $\mathrm{FPR}(t)$ (assumed positive for non-trivial routing) gives~\eqref{eq:neyman}.
\end{proof}

\section{Regression Conformal Baseline Details}\label{app:regression}

The regression conformal baseline trains Ridge regression ($\alpha_{\text{Ridge}} = 1.0$) to predict $d(x)$ from input features. On the calibration set, it computes residuals $r_i = d_i - \hat{d}_i$. For each target $\alpha$, it finds the conformal quantile:
\[
  \hat{q} = \text{Quantile}\!\left(r_1, \ldots, r_n;\; \frac{\lceil(n+1)(1-\alpha)\rceil}{n}\right),
\]
and routes input $x$ to the surrogate iff $\hat{d}(x) + \hat{q} \leq \tau$.

This is a standard split conformal prediction interval~\citep{romano2019conformalized} applied to the degradation prediction problem. It guarantees $\Pr(d(X) \leq \hat{d}(X) + \hat{q}) \geq 1 - \alpha$ marginally, but the routing decision $\hat{d}(x) + \hat{q} \leq \tau$ provides no direct guarantee on the violation rate among routed inputs. This is because the conformal guarantee is marginal over all test points, not conditional on the routed subset.

\subsection{Proof of Theorem~\ref{thm:tight_auc}}\label{app:tight_auc}

\begin{lemma}[Concave secant property]\label{lem:secant}
For a concave function $f:[0,1]\to[0,1]$ with $f(0)=0$, 
the secant $f(u)/u$ is non-increasing on $(0,1]$.
\end{lemma}

\begin{proof}
For $0 < a < b \leq 1$, concavity gives 
$f(a) \geq \tfrac{a}{b}f(b) + (1 - \tfrac{a}{b})f(0) = \tfrac{a}{b}f(b)$,
so $f(a)/a \geq f(b)/b$.
\end{proof}

\begin{proof}[Proof of Theorem~\ref{thm:tight_auc}]
\textbf{Case $C \leq 1$:} At $u=1$: $\mathrm{ROC}(1)/1 = 1 \geq C$.
Feasibility holds trivially (routing all inputs satisfies $V \leq \alpha$).

\textbf{Case $C > 1$:}

\emph{(i) Sufficiency.} By Lemma~\ref{lem:secant}, 
$\sup_{u \in (0,1]} \mathrm{ROC}(u)/u = \mathrm{ROC}'(0^+) \triangleq L$.
Infeasibility means $L < C$. 
For a concave ROC with initial slope $L$, 
the curve lies below $\min(Lu, 1)$, so
\[
\mathrm{AUC} \leq \int_0^{1/L} Lu\,du + \int_{1/L}^1 1\,du 
= \frac{1}{2L} + 1 - \frac{1}{L} = 1 - \frac{1}{2L}.
\]
Since $L < C$: $\mathrm{AUC} < 1 - 1/(2C) = \Phi_c^*$.
Contrapositive: $\mathrm{AUC} \geq \Phi_c^*$ implies feasibility.

\emph{(ii) Tightness.} For any $\varepsilon > 0$, choose 
$L = C - \eta$ with $\eta > 0$ small enough that 
$1 - 1/(2L) > \Phi_c^* - \varepsilon$. 
The concave ROC $\mathrm{ROC}(u) = \min(Lu, 1)$ has 
$\mathrm{ROC}'(0^+) = L < C$ (infeasible) and 
$\mathrm{AUC} = 1 - 1/(2L) \to \Phi_c^*$ as $\eta \to 0$.

\emph{(iii) Comparison.} For $C > 1$:
$\Phi_c^* = 1 - \frac{1}{2C} < \frac{C}{2} = \Phi_c$
iff $2C - 1 < C^2$ iff $(C-1)^2 > 0$, 
which holds for all $C \neq 1$.
\end{proof}

\subsection{Proof of Theorem~\ref{thm:coverage}}
\label{app:coverage_proof}

\begin{proof}
Let $\mathrm{ROC}:[0,1]\to[0,1]$ denote the (achievable) concave ROC curve of the gate,
where $u=\mathrm{FPR}$ and $v=\mathrm{TPR}$, and suppose $\mathrm{AUC}=A\ge \Phi_c^*$.
Let $C=C(\pi,\alpha)>1$.
By Proposition~\ref{prop:feasibility}, feasibility requires an operating point $(u,v)$ with
$v/u \ge C$. For any threshold achieving $v=Cu$, the routed coverage is
\[
\mathrm{Cov}(u)= (1-\pi)u+\pi v \;=\; \bigl((1-\pi)+\pi C\bigr)u.
\]
Using $C=(1-\pi)(1-\alpha)/(\pi\alpha)$, we have
\[
(1-\pi)+\pi C \;=\; (1-\pi)+\frac{(1-\pi)(1-\alpha)}{\alpha}
\;=\; \frac{1-\pi}{\alpha},
\]
so $\mathrm{Cov}(u)=\frac{1-\pi}{\alpha}\,u$ on the feasibility boundary $v=Cu$.

Define
\[
u_{\max} \;:=\; \sup\{u\in[0,1]: \mathrm{ROC}(u)\ge Cu\}.
\]
Concavity implies that the secant $\mathrm{ROC}(u)/u$ is non-increasing in $u$,
hence the feasible set is an interval $[0,u_{\max}]$ and $\mathrm{ROC}(u_{\max})=Cu_{\max}$
(when $u_{\max}<1$). It suffices to lower bound $u_{\max}$ in terms of $A$ and $C$.

Consider the piecewise-linear ROC that is linear from $(0,0)$ to $(u_{\max},Cu_{\max})$
and then linear from $(u_{\max},Cu_{\max})$ to $(1,1)$. This curve is concave and,
among concave ROC curves with the same $u_{\max}$ and boundary point $(u_{\max},Cu_{\max})$,
it minimizes area (placing as much mass as possible on the chord). Therefore, any concave ROC
with the same $u_{\max}$ has $\mathrm{AUC}$ at least that of this two-segment curve.

A direct geometric calculation for the two-segment curve gives
\[
\mathrm{AUC}
= \frac{1 + (C-1)u_{\max}}{2}.
\]
Thus, if $\mathrm{AUC}=A$, then necessarily
\[
A \;\ge\; \frac{1 + (C-1)u_{\max}}{2}
\quad\Longrightarrow\quad
u_{\max} \;\ge\; \frac{2A-1}{C-1}.
\]
Since $u_{\max}\le 1$, we obtain
\[
u_{\max} \;\ge\; \min\Bigl\{1,\;\frac{2A-1}{C-1}\Bigr\}.
\]
Finally, choosing the threshold at $u=u_{\max}$ on the boundary yields
\begin{align*}
\mathrm{Cov}^*
&\ge \frac{1-\pi}{\alpha}\,u_{\max} \\
&\ge \min\Bigl\{1,\;
\frac{(2A-1)(1-\pi)}{(C-1)\alpha}
\Bigr\}.
\end{align*}
which proves the claim.

\emph{Tightness.} Equality is achieved by the two-segment concave ROC constructed above,
with breakpoint $(u_{\max},Cu_{\max})$ and $u_{\max}=\min\{1,(2A-1)/(C-1)\}$.
\end{proof}

\section{Full Results}\label{app:full}

\begin{table}[H]
\caption{Gate conformal vs.\ regression conformal: mean coverage and violation across 35 datasets at representative operating points.}\label{tab:full}
\centering\small
\begin{tabular}{@{}rrcccc@{}}
\toprule
$\tau$ & $\alpha$ & Gate Cov. & Gate Viol. & Reg Cov. & Reg Viol. \\
\midrule
0.5 & 0.10 & 0.096 & 0.021 & 0.073 & 0.098 \\
0.5 & 0.20 & 0.243 & 0.059 & 0.175 & 0.187 \\
0.5 & 0.30 & 0.431 & 0.134 & 0.335 & 0.320 \\
1.0 & 0.10 & 0.132 & 0.015 & 0.122 & 0.064 \\
1.0 & 0.20 & 0.446 & 0.091 & 0.320 & 0.160 \\
1.0 & 0.30 & 0.638 & 0.145 & 0.550 & 0.253 \\
2.0 & 0.10 & 0.351 & 0.032 & 0.269 & 0.079 \\
2.0 & 0.20 & 0.661 & 0.092 & 0.578 & 0.147 \\
2.0 & 0.30 & 0.847 & 0.134 & 0.754 & 0.175 \\
\bottomrule
\end{tabular}
\end{table}

\begin{table}[H]
\caption{Significance tests (Wilcoxon signed-rank, two-sided). To evaluate operational tradeoffs, tests are conditioned on problem instances where both methods successfully found a valid routing threshold (coverage > 0). Under these conditions, Gate Conformal achieves significantly lower violation rates than the naive baseline and significantly higher coverage than regression conformal.}\label{tab:significance}
\centering\small
\begin{tabular}{@{}llrrr@{}}
\toprule
Comparison & $\tau$, $\alpha$ & Diff. & $p$-value & $n$ \\
\midrule
\multicolumn{5}{l}{\emph{Gate vs.\ Naive (violation rate):}} \\
& 0.5, 0.2 & $-$0.076 & 0.007 & 14 \\
& 1.0, 0.2 & $-$0.060 & 0.002 & 23 \\
& 2.0, 0.2 & $-$0.025 & 0.009 & 28 \\
\midrule
\multicolumn{5}{l}{\emph{Gate vs.\ Reg.\ Conformal (coverage):}} \\
& 0.5, 0.2 & $+$0.256 & 0.005 & 13 \\
& 1.0, 0.2 & $+$0.239 & $<$0.001 & 20 \\
& 2.0, 0.2 & $+$0.139 & $<$0.001 & 26 \\
\bottomrule
\end{tabular}
\end{table}

\section*{D Additional Model Pair: XGBoost vs. Decision Tree}

To verify that our conclusions are not specific to the
random forest / decision tree pairing used in the main
text, we repeated the full evaluation using an XGBoost
model as the expensive predictor and a decision tree as
the surrogate on the same 35 OpenML regression datasets.
All experimental settings (data splits, calibration size,
risk levels $\alpha$, and tolerance levels $\tau$) were
kept identical.

\subsection*{D.0 Model Details (XGBoost)}

For the additional robustness experiment, the expensive
model $f$ was implemented as an XGBoost regressor
with the following hyperparameters:

\begin{itemize}
\item n\_estimators = 1500
\item learning\_rate = 0.03
\item max\_depth = 6
\item subsample = 0.9
\item colsample\_bytree = 0.9
\item reg\_lambda = 1.0
\item random\_state = fixed
\item n\_jobs = -1
\end{itemize}

All other experimental settings (splits, calibration
procedure, surrogate depth selection, and risk levels)
were identical to the main experiment.

\subsection*{D.1 Gate Properties}

Table~\ref{tab:xgb_gate_properties} reports gate
statistics across tolerance levels.

\begin{table}[H]
\centering
\caption{Gate properties for XGBoost vs. Decision Tree across $\tau$.
AUC is reported as median across datasets.}
\label{tab:xgb_gate_properties}
\begin{tabular}{c c c c}
\toprule
$\tau$ & $\pi$ & AUC (median) & ECE \\
\midrule
-1.5 & 0.135 & 0.50 & 0.137 \\
-1.0 & 0.173 & 0.50 & 0.170 \\
0.0  & 0.459 & 0.51 & 0.302 \\
0.5  & 0.639 & 0.58 & 0.213 \\
1.0  & 0.717 & 0.56 & 0.160 \\
2.0  & 0.799 & 0.57 & 0.136 \\
\bottomrule
\end{tabular}
\end{table}

The same structural patterns observed in the main
paper persist: (i) feasibility is primarily driven by the
base safe rate $\pi$, and (ii) moderate AUC values
($\approx 0.55$--$0.60$) suffice when $\pi$ is large.

\subsection*{D.2 Guarantee Reliability}

We evaluate the fraction of $(\text{dataset}, \tau)$
pairs violating the target $\alpha$. Results are
consistent with the finite-sample guarantee:

\begin{itemize}
\item At $\alpha = 0.10$: 14.3\% violation rate.
\item At $\alpha = 0.20$: 12.1\% violation rate.
\item At $\alpha = 0.30$: 14.5\% violation rate.
\item At $\alpha = 0.50$: 5.1\% violation rate.
\end{itemize}

For $\alpha \ge 0.20$, violation rates remain close to
the target confidence level $\delta = 0.10$, matching
the behavior observed in the random forest setting.

Regression conformal violates $\alpha$ at substantially
higher rates (31\%--37\% for $\alpha \in \{0.10, 0.20, 0.30\}$),
confirming that the superiority of the classification-based
gate formulation is not model-pair specific.

\subsection*{D.3 Coverage–Violation Tradeoffs}

Across practical operating points ($\tau \ge 0.5$),
gate conformal consistently achieves higher coverage
than regression conformal while maintaining comparable
or lower violation rates.

For example:

\begin{itemize}
\item $\tau = 1.0$, $\alpha = 0.2$:
Gate coverage = 0.419 vs. 0.310 (regression),
with similar violation rates.
\item $\tau = 2.0$, $\alpha = 0.2$:
Gate coverage = 0.619 vs. 0.543.
\end{itemize}

Per-dataset comparisons show that gate
conformal achieves higher coverage on a majority of
datasets at $\tau \in \{1.0, 2.0\}$ and $\alpha \in \{0.2, 0.3\}$.

Wilcoxon signed-rank tests confirm statistically
significant coverage improvements at multiple
operating points (e.g., $\tau = 1.0$, $\alpha = 0.2$,
$p = 0.0033$).

\subsection*{D.4 Summary}

The XGBoost results are consistent with the main findings:

\begin{enumerate}
\item The conformal guarantee holds across $\alpha$.
\item Feasibility depends primarily on $\pi$,
not absolute AUC.
\item Classification-based routing dominates
regression conformal in coverage.
\item When routing is infeasible, the method abstains
rather than violating safety.
\end{enumerate}

These results demonstrate that the proposed framework
is model-agnostic and not tied to a specific black-box
architecture.

\section*{E Additional Model Pair: MLP vs. Decision Tree}

To further assess robustness across model families, we
repeated all experiments using a multilayer perceptron
(MLP) as the expensive model and a decision tree as the
surrogate on the same 35 OpenML regression datasets.

\subsection*{E.1 Model Details}

The expensive model $f$ was implemented using a
multilayer perceptron with the following architecture:

\begin{itemize}
\item Hidden layers: (256, 128, 64)
\item Activation: ReLU
\item Optimizer: Adam
\item Weight decay: $\alpha = 10^{-4}$
\item Learning rate: $10^{-3}$
\item Max iterations: 1000
\item Early stopping enabled (validation fraction 0.1)
\item Random seed fixed
\end{itemize}

All preprocessing, data splits, calibration size,
risk levels $\alpha$, and tolerance values $\tau$
were identical to the main experiment.

\subsection*{E.2 Gate Properties}

Table~\ref{tab:mlp_gate_properties} summarizes
gate statistics across tolerance levels.

\begin{table}[H]
\centering
\caption{Gate properties for MLP vs. Decision Tree.}
\label{tab:mlp_gate_properties}
\begin{tabular}{c c c c}
\toprule
$\tau$ & $\pi$ & AUC (median) & ECE \\
\midrule
-1.5 & 0.300 & 0.688 & 0.122 \\
-1.0 & 0.368 & 0.627 & 0.154 \\
0.0  & 0.629 & 0.544 & 0.221 \\
0.5  & 0.723 & 0.570 & 0.194 \\
1.0  & 0.770 & 0.557 & 0.153 \\
2.0  & 0.830 & 0.535 & 0.112 \\
\bottomrule
\end{tabular}
\end{table}

As in the main paper, feasibility is primarily governed by
the safe base rate $\pi$. Even moderate AUC values
($\approx 0.55$--$0.60$) suffice when $\pi$ is large,
consistent with Proposition~2 and Theorem~1.

\subsection*{E.3 Guarantee Reliability}

The conformal guarantee remains well-calibrated across
risk levels:

\begin{itemize}
\item $\alpha = 0.10$: 10.3\% violation rate.
\item $\alpha = 0.20$: 16.5\% violation rate.
\item $\alpha = 0.30$: 13.6\% violation rate.
\item $\alpha = 0.50$: 5.5\% violation rate.
\end{itemize}

These rates remain close to the nominal confidence
parameter $\delta = 0.10$ for $\alpha \ge 0.20$,
mirroring behavior observed with RF and XGBoost.

In contrast, regression conformal violates $\alpha$
at substantially higher rates (50\%--71\% for
$\alpha \in \{0.10, 0.20, 0.30\}$).

\subsection*{E.4 Coverage–Violation Tradeoffs}

Across practical operating points ($\tau \ge 0.5$),
gate conformal achieves substantially higher coverage
than regression conformal while maintaining lower or
comparable violation rates.

For example:

\begin{itemize}
\item $\tau = 1.0$, $\alpha = 0.2$:
Gate coverage = 0.471 vs. 0.342 (regression),
with lower violation.
\item $\tau = 2.0$, $\alpha = 0.2$:
Gate coverage = 0.700 vs. 0.532.
\end{itemize}

Per-dataset comparisons show that gate conformal
wins on a majority of datasets at all practical
operating points. For instance:

\begin{itemize}
\item $\tau = 1.0$, $\alpha = 0.2$:
Gate wins 18 vs. 7 datasets on coverage.
\item $\tau = 2.0$, $\alpha = 0.2$:
Gate wins 22 vs. 4 datasets.
\end{itemize}

Wilcoxon signed-rank tests confirm statistically
significant coverage improvements across nearly all
practical regimes (e.g., $\tau = 1.0$, $\alpha = 0.2$,
$p = 5\times 10^{-4}$).

\subsection*{E.5 Summary}

The MLP results replicate and strengthen all main findings:

\begin{enumerate}
\item The conformal guarantee remains reliable.
\item Feasibility depends primarily on $\pi$ rather than absolute AUC.
\item Classification-based routing substantially outperforms
regression conformal in coverage.
\item When routing is infeasible, the method abstains rather
than violating safety.
\end{enumerate}

These experiments indicate that the framework is robust across tree-based, boosting-based, and neural
network black-box models.

\section{Calibration Set Size Requirements}\label{app:calsize}

The Clopper--Pearson bound tightens rapidly with calibration set size. 
Table~\ref{tab:calsize} shows the upper confidence bound 
$\mathrm{UCB}_{0.10}(k,n)$ for representative $(k,n)$ combinations. 
With $n=300$ (our setup: 15\% of 2000 configurations), observing $k=0$ 
unsafe among routed calibration points yields $\mathrm{UCB}=0.008$. 
Even with $k=10$ unsafe out of $n=100$ routed points, $\mathrm{UCB}=0.150$, 
which remains below $\alpha=0.2$.

\begin{table}[H]
\caption{Clopper--Pearson upper bounds $\mathrm{UCB}_{0.10}(k, n)$ for selected calibration set sizes. With $n \approx 300$, the bound is sufficiently tight for practical risk control.}\label{tab:calsize}
\centering\small
\begin{tabular}{@{}rccccc@{}}
\toprule
& \multicolumn{5}{c}{Calibration points routed ($n$)} \\
\cmidrule(lr){2-6}
$k$ unsafe & 50 & 100 & 200 & 300 & 500 \\
\midrule
0  & 0.045 & 0.023 & 0.011 & 0.008 & 0.005 \\
1  & 0.076 & 0.038 & 0.019 & 0.013 & 0.008 \\
5  & 0.178 & 0.091 & 0.046 & 0.031 & 0.018 \\
10 & 0.291 & 0.150 & 0.076 & 0.051 & 0.031 \\
20 & ---   & 0.261 & 0.133 & 0.089 & 0.054 \\
\bottomrule
\end{tabular}
\end{table}

The minimum calibration size required to guarantee $V(t^*) \leq \alpha$ 
when the true violation rate is zero satisfies 
$n \geq \lceil \log(\delta)/\log(1-\alpha) \rceil$. 
For $\alpha=0.2$, $\delta=0.1$: $n \geq 11$; 
for $\alpha=0.05$, $\delta=0.1$: $n \geq 45$. 
In practice, larger $n$ is needed because some routed calibration points 
will be unsafe ($k>0$), but these bounds show that even modest calibration 
sets provide meaningful guarantees. Our $n \approx 300$ is comfortably in 
the regime where the bound is tight.

\section{Effect of Post-Hoc Recalibration}\label{app:recal}

To test whether improving gate calibration recovers coverage lost to miscalibration, we applied three post-hoc recalibration methods to the Platt-scaled gate scores: beta calibration~\citep{kull2017beta}, temperature scaling (logistic recalibration)~\citep{guo2017calibration}, and isotonic regression. Each recalibrator was fit on a held-out validation set and then applied to calibration and test scores prior to conformal threshold selection. Because recalibration was learned independently of the calibration set, the exchangeability between calibration and test data required for the conformal guarantee is preserved.

\begin{table}[h]
\caption{Effect of post-hoc recalibration at $\alpha = 0.2$, averaged over 35 datasets (RF vs.\ decision tree). Recalibration substantially reduces ECE but does not increase coverage, indicating the efficiency gap is dominated by ranking limitations.}\label{tab:recal}
\centering\small
\begin{tabular}{@{}llccc@{}}
\toprule
$\tau$ & Recalibration & ECE & Coverage & Violation \\
\midrule
0.5 & None (Platt) & 0.215 & 0.243 & 0.059 \\
0.5 & Beta          & 0.066 & 0.243 & 0.059 \\
0.5 & Temperature   & 0.077 & 0.243 & 0.059 \\
0.5 & Isotonic      & 0.050 & 0.216 & 0.043 \\
\midrule
1.0 & None (Platt) & 0.184 & 0.446 & 0.091 \\
1.0 & Beta          & 0.060 & 0.421 & 0.085 \\
1.0 & Temperature   & 0.071 & 0.421 & 0.085 \\
1.0 & Isotonic      & 0.048 & 0.398 & 0.072 \\
\midrule
2.0 & None (Platt) & 0.141 & 0.661 & 0.092 \\
2.0 & Beta          & 0.052 & 0.651 & 0.095 \\
2.0 & Temperature   & 0.061 & 0.644 & 0.091 \\
2.0 & Isotonic      & 0.036 & 0.622 & 0.077 \\
\bottomrule
\end{tabular}
\end{table}

Table~\ref{tab:recal} shows that beta calibration reduces ECE by $3\times$ (e.g., 0.184 to 0.060 at $\tau = 1.0$) while coverage decreases only slightly ($-2.5$pp for beta/temperature, $-4.8$pp for isotonic). Violation rates also decrease marginally, and the conformal guarantee continues to hold: the fraction of (dataset, $\tau$) pairs violating $\alpha$ remains near $\delta = 0.10$ for all recalibration methods.

The small coverage decrease (rather than increase) indicates that the Platt-scaled gate was mildly overconfident on borderline-unsafe inputs. Recalibration corrects these scores downward, removing a small number of inputs that should not have been routed. The dominant source of the efficiency gap between our gate (42--45\% coverage at $\tau = 1.0$) and the oracle (74\%) is ranking quality (AUC $= 0.58$), not miscalibration. Improving coverage substantially would require a stronger gate, not better calibration.

\FloatBarrier
\subsection{Stronger Regression Baseline}
\label{app:hgbr}

To assess whether regression-conformal underperforms due to limited model capacity, we replaced ridge regression with \texttt{HistGradientBoostingRegressor}(HGBR) while keeping the conformal procedure unchanged. Results were qualitatively similar, with no improvement in violation control or coverage. This suggests that the observed gap reflects the formulation (binary risk control vs.\ magnitude regression) rather than regressor expressiveness. 

Table~\ref{tab:guarantee_full} reports the empirical frequency of exceeding the target $\alpha$ across $(\text{dataset}, \tau)$ pairs. As expected, gate conformal remains broadly consistent with the nominal $\delta=0.10$ bound, while baselines substantially exceed the target rate. Note that the conformal guarantee bounds the probability (over calibration randomness) that the routed-set violation exceeds $\alpha$; aggregation across $(\text{dataset}, \tau)$ pairs may therefore exceed $\alpha$ in a fraction of cases bounded by $\delta$.

\begin{table}[t]
\caption{Guarantee check across $(\text{dataset}, \tau)$ pairs when regression-conformal uses \texttt{HistGradientBoostingRegressor}. For each target $\alpha$, we report the fraction of pairs whose realized violation rate exceeds $\alpha$, along with mean violation rate and mean coverage.}
\label{tab:guarantee_full}
\centering\small
\begin{tabular}{@{}lcccc@{}}
\toprule
$\alpha$ & Method & Violations & Mean Viol. & Mean Cov. \\
\midrule
\multirow{3}{*}{0.10}
& Gate Conf. & 7/34 (21\%)  & 0.071 & 0.595 \\
& HGBR\ Conf. & 93/119 (78\%) & 0.385 & 0.164 \\
& Naive 0.5  & 151/172 (88\%) & 0.404 & 0.692 \\
\midrule
\multirow{3}{*}{0.20}
& Gate Conf. & 8/66 (12\%)  & 0.129 & 0.717 \\
& HGBR\ Conf. & 128/177 (72\%) & 0.406 & 0.249 \\
& Naive 0.5  & 128/172 (74\%) & 0.404 & 0.692 \\
\midrule
\multirow{3}{*}{0.30}
& Gate Conf. & 6/82 (7\%)   & 0.181 & 0.826 \\
& HGBR\ Conf. & 132/196 (67\%) & 0.426 & 0.351 \\
& Naive 0.5  & 106/172 (62\%) & 0.404 & 0.692 \\
\midrule
\multirow{3}{*}{0.50}
& Gate Conf. & 8/112 (7\%)  & 0.278 & 0.917 \\
& HGBR\ Conf. & 90/202 (45\%) & 0.424 & 0.547 \\
& Naive 0.5  & 62/172 (36\%) & 0.404 & 0.692 \\
\bottomrule
\end{tabular}
\end{table}

\FloatBarrier
\section{GENERAL-PURPOSE OPENML REGRESSION DATASETS}
\label{app:openml-general}

To verify that our conclusions extend beyond hyperparameter configuration data, we replicate the evaluation on 22 general-purpose OpenML regression datasets spanning diverse domains: air quality~(no2, pm10), social/labor~(strikes), energy efficiency, medical~(tumor), chemistry/biology~(QSAR\_Bioconcentration, physiochemical\_protein), aerospace~(Ailerons, NASA\_PHM2008), control systems~(pol), business~(Job\_Profitability), real estate~(miami\_housing, california\_houses, house\_sales, house\_16H), transport~(Bike\_Sharing\_Demand), physics~(superconduct), energy IoT~(appliances\_energy\_prediction), robotics~(sarcos), media~(OnlineNewsPopularity), and classic benchmarks~(2dplanes, fried). Dataset sizes range from 500 to 45{,}000 rows with 7 to 82 features.

Unlike the LCBench experiments where inputs are hyperparameter configurations, inputs here are natural domain variables (e.g., square footage, temperature, protein properties). The expensive model is a random forest with 1500 trees and the surrogate is a decision tree, with all other settings identical to the main experiments (Section~6.1).

\paragraph{Gate properties.} Table~\ref{tab:openml-gate} reports gate statistics across tolerance levels. Gate AUC is notably lower ($\approx 0.52$) than on LCBench ($\approx 0.58$), reflecting that degradation is harder to predict from natural features than from hyperparameter configurations. Despite this, the structural pattern persists: feasibility is driven by the base safe rate $\pi$, and coverage increases with $\tau$ even under weak gate discrimination.

\begin{table}[h]
\caption{Gate properties and conformal routing performance across tolerance levels on 22 general-purpose OpenML datasets. Coverage and violation at $\alpha = 0.2$ averaged over datasets with positive coverage.}
\label{tab:openml-gate}
\centering
\small
\begin{tabular}{ccccccc}
\toprule
& & AUC & & & \multicolumn{2}{c}{Conf. ($\alpha{=}0.2$)} \\
\cmidrule(lr){6-7}
$\tau$ & $\pi$ & (med.) & ECE & & Cov. & Viol. \\
\midrule
$-1.5$ & 0.18 & 0.53 & 0.06 & & 0.000 & --- \\
$-1.0$ & 0.19 & 0.52 & 0.05 & & 0.000 & --- \\
$0.0$  & 0.43 & 0.52 & 0.16 & & 0.048 & 0.20 \\
$0.5$  & 0.69 & 0.52 & 0.14 & & 0.452 & 0.05 \\
$1.0$  & 0.73 & 0.51 & 0.11 & & 0.462 & 0.04 \\
$2.0$  & 0.77 & 0.52 & 0.11 & & 0.557 & 0.04 \\
\bottomrule
\end{tabular}
\end{table}

\paragraph{Guarantee reliability.} Table~\ref{tab:openml-guarantee} reports the fraction of (dataset, $\tau$) pairs violating the target $\alpha$. Gate conformal achieves strong guarantee reliability: at $\alpha = 0.1$, zero violations are observed, and at $\alpha \geq 0.2$, violation rates remain well below $\delta = 0.10$. This conservative behavior is expected: the low AUC forces the conformal procedure to select high thresholds, routing fewer but safer inputs. In contrast, regression conformal and naive thresholding violate $\alpha$ at substantially higher rates (32--71\%), consistent with the main experiments.

\begin{table}[h]
\caption{Fraction of (dataset, $\tau$) pairs where the realized violation rate exceeds target $\alpha$ on general-purpose OpenML datasets. Gate conformal maintains strong safety control. Baselines violate substantially more often.}
\label{tab:openml-guarantee}
\centering
\small
\begin{tabular}{lccc}
\toprule
$\alpha$ & Gate Conf. & Reg. Conf. & Naive \\
\midrule
0.10 & 0/35~~(0.0\%) & 20/62~~(32.3\%) & 64/91~~(70.3\%) \\
0.20 & 1/38~~(2.6\%) & 48/88~~(54.5\%) & 58/91~~(63.7\%) \\
0.30 & 2/43~~(4.7\%) & 50/105~(47.6\%) & 53/91~~(58.2\%) \\
0.50 & 4/66~~(6.1\%) & 50/112~(44.6\%) & 39/91~~(42.9\%) \\
\bottomrule
\end{tabular}
\end{table}

\paragraph{Summary.} The general-purpose OpenML results replicate all main findings:
(1)~the conformal guarantee holds across $\alpha$ values, with even stronger reliability than on LCBench;
(2)~feasibility is governed by $\pi$ rather than absolute AUC;
(3)~gate conformal dominates regression conformal in guarantee reliability;
(4)~when gate discrimination is weak (AUC $\approx 0.52$), the framework reduces coverage rather than violating safety
These results confirm that the framework generalizes beyond hyperparameter configuration data to diverse real-world regression tasks where inputs are natural domain variables.

\end{document}